\documentclass[conference,letterpaper]{IEEEtran}

\usepackage[utf8]{inputenc} 
\usepackage[T1]{fontenc}
\usepackage{url}
\usepackage{ifthen}
\usepackage{cite}
\usepackage[cmex10]{amsmath} 

\usepackage[shortlabels]{enumitem}
\usepackage{graphicx}
\usepackage{epsfig}
\usepackage{times}
\usepackage{amssymb}
\usepackage{color}
\usepackage{float}
\usepackage{multirow}
\usepackage{color}
\renewcommand{\baselinestretch}{1}
\usepackage{enumitem}
\usepackage[normalem]{ulem}
\usepackage{url}
\usepackage{comment,wrapfig,subcaption}
\usepackage[font=small]{caption}
\usepackage{amsthm}
\usepackage{wrapfig}
\usepackage{hyperref}
\usepackage{mathtools}

\usepackage{titlesec}

\newcommand{\prob}[1]{\mathbb{P}\left[{#1}\right]}

\newcommand{\beq}{\begin{equation}}
\newcommand{\eeq}{\end{equation}}
\newcommand{\beqn}{\begin{equation*}}
\newcommand{\eeqn}{\end{equation*}}
\newcommand{\beqa}{\begin{eqnarray}}
\newcommand{\eeqa}{\end{eqnarray}}
\newcommand{\beqan}{\begin{eqnarray*}}
\newcommand{\eeqan}{\end{eqnarray*}}

\newcommand{\lep}[1]{\mathop  \le \limits^{(#1)}}

\newcommand{\ep}[1]{\mathop  = \limits^{(#1)}}

\newcommand{\Real}{\mathbb{R}}

\newcommand{\norm}[1]{\left\lVert #1 \right\rVert}
\DeclareMathOperator*{\argmax}{arg\,max} %

\newcommand{\inner}[2]{\langle #1, #2 \rangle}
\newcommand{\ignore}[1]{}

\newcommand{\ex}[1]{\mathbb{E}\left[ #1 \right] }
\setlist[description]{
  leftmargin=\parindent, itemindent=-.5em,
  itemsep=-0.4em, topsep=.5em
}

\makeatletter
\def\thm@space@setup{\thm@preskip=2pt
\thm@postskip=0pt}
\makeatother

\newtheorem{thm}{Theorem}
\newtheorem{thm*}{Theorem}[section]

\newtheorem{cor*}{Corollary}[section]

\newtheorem*{corollary*}{Corollary}
\newtheorem{lemma}{Lemma}
\newtheorem{claim}{Claim}
\newtheorem*{lemma*}{Lemma}

\newtheorem*{prop*}{Proposition}

\usepackage[ruled,linesnumbered,vlined]{algorithm2e}
\SetKwInOut{Parameter}{parameter}

\titlespacing*{\section}{0pt}{0.75\baselineskip}{0.25\baselineskip}
\titlespacing*{\subsection}{0pt}{0.5\baselineskip}{0.25\baselineskip}

\begin{document}

\setlength{\belowdisplayskip}{3pt} \setlength{\belowdisplayshortskip}{3pt}
\setlength{\abovedisplayskip}{3pt} \setlength{\abovedisplayshortskip}{3pt}

\title{No-Regret Algorithms for Time-Varying Bayesian Optimization} 


\author{%
  \IEEEauthorblockN{Xingyu Zhou}
  \IEEEauthorblockA{Wayne State University\\
                    Detroit, USA \\ 
                    Email: xingyu.zhou@wayne.edu}
   \and
  \IEEEauthorblockN{Ness Shroff}
  \IEEEauthorblockA{The Ohio State University\\
                    Columbus, USA \\ 
                    Email: shroff.11@osu.edu}
}


\maketitle
\begingroup\renewcommand\thefootnote{\IEEEauthorrefmark{1}}

\endgroup
\begin{abstract}
In this paper, we consider the time-varying Bayesian optimization problem. The unknown function at each time is assumed to lie in an RKHS (reproducing kernel Hilbert space) with a bounded norm. We adopt the general variation budget model to capture the time-varying environment, and the variation is  characterized by the change of the RKHS norm. We adapt the restart and sliding window mechanism to introduce two GP-UCB type algorithms: R-GP-UCB and SW-GP-UCB, respectively. We derive the first (frequentist) regret guarantee on the dynamic regret for both algorithms. Our results not only recover previous linear bandit results when a linear kernel is used, but complement the previous regret analysis of time-varying Gaussian process bandit under a Bayesian-type  regularity assumption, i.e., each function is a sample from a Gaussian process.
\end{abstract}

\section{Introduction}
We consider the online black-box optimization of an unknown function $f$ with only bandit feedback. At each time $t$, one query point $x_t$ is selected and the (noisy) corresponding evaluation (reward) $y_t$ is then observed. The objective is to find the optimal point using a minimal number of trials. This model is ubiquitous in bandit learning problems under various assumptions. For instance, when the domain is a finite set with independent actions (points), this is the typical  multi-arm bandit (MAB) problem~\cite{lai1985asymptotically}. If the unknown reward function is linear, this represents the standard linear bandit problem~\cite{abbasi2011improved}. In Bayesian optimization, the unknown function could be arbitrary (e.g., non-linear and non-convex). Under a Bayesian-type regularity assumption, the unknown function is assumed to be a sample from a Gaussian process (and hence also often called Gaussian process bandit)~\cite{srinivas2009gaussian}. Another popular regularity assumption is a frequentist-type view. That is, the unknown function is a fixed function in an RKHS with a bounded norm~\cite{chowdhury2017kernelized}.

This versatile model of online decision making has found many successful real-life applications, such as recommendation system, dynamic pricing, and network resource allocation. The standard application of this model often assumes that the unknown function $f$ is fixed, which is, however, typically not the case in real-life scenarios. For instance, in the dynamic pricing problem, $f$ often represents the market environment, which definitely evolves with time. Another intriguing example is the dynamic resource allocation in wireless communications, where the channel conditions change with time due to fading.

To this end, non-stationary bandit learning has recently drawn significant interest. In MAB, a variety of works have studied either abrupt changes or slowly changing distributions\cite{auer2019adaptively,besbes2019optimal}. In the stochastic linear bandit setting, \cite{cheung2019learning} takes the first attempt by considering a general \emph{variation budget} model (i.e., $\sum_{t=1}^{T-1} \norm{\theta_{t} -\theta_{t+1} } \le P_T$), which is able to capture both abruptly-changing and slowly-changing environments. Under this model, the authors propose a sliding window UCB-type algorithm. Later, under the same model,~\cite{russac2019weighted} proposes a weighted-UCB algorithm and~\cite{zhao2020simple} introduces a simple restarting UCB algorithm. Recently,~\cite{zhao2021non} has pointed out a key common gap in the analysis of the three works above on linear bandits (i.e.,~\cite{cheung2019learning,russac2019weighted,zhao2020simple}). After fixing the gap, the regret bound of all three algorithms is on the order of $\widetilde{O}\left(d^{7/8}T^{3/4} (1+P_T)^{1/4}\right)$, where $d$ is the dimension of the action space and $P_T$ is the total variation budget. 
In the case of a general non-linear function,~\cite{bogunovic2016time} considers the GP bandit setting where each function $f_t$ is a sample from a GP (i.e., Bayesian-type regularity) and establishes both lower and upper bounds. However, Bayesian optimization under a  frequentist-type regularity with time-varying functions remains untouched. In fact, it was stated as one interesting future direction to explore in~\cite{bogunovic2016time}.

Motivated by this, in this paper, we consider a time-varying Bayesian optimization under a frequentist-type regularity. That is, the smoothness of the unknown time-varying functions $f_t$ is indicated by the corresponding RKHS norm. The time-varying environment is also captured by the RKHS norm by $\sum_{t=1}^{T-1} \norm{f_{t} -f_{t+1} }_{\mathcal{H}} \le P_T$. Under this regularity assumption and time-varying model, we made the following contributions.
\begin{itemize}
    \item We adapt the restarting-UCB and sliding window UCB algorithm in linear bandits to the GP bandit setting. We call them R-GP-UCB\footnote{Note that R-GP-UCB was first proposed in~\cite{bogunovic2016time} and analyzed in the Bayesian regularity condition.} and SW-GP-UCB.
    We derive the first (frequentist) regret guarantee on the dynamic regret of both algorithms. For a known $P_T$, both enjoy regret bound $\widetilde{O}(\gamma_T^{7/8} (1+P_T)^{1/4} T^{3/4})$, and when $P_T$ is unknown, the regret bound is $\widetilde{O}(\gamma_T^{7/8} (1+P_T) T^{3/4})$, where $\gamma_T$ is the maximum information gain.
    \item Our regret bounds include previous linear bandit (e.g.,\cite{cheung2019learning,russac2019weighted,zhao2020simple}) results as special case by choosing a linear kernel for the RKHS, under which $\gamma_T = O(d\ln T)$. Moreover, our results also complement the regret bounds under the Bayesian-type regularity assumptions in~\cite{bogunovic2016time}.
\end{itemize}

\section{Problem Statement and Preliminaries}
We consider the time-varying Bayesian optimization using the sequential decision-making model. Under this model, in each time $t=1,2,\ldots$, the learning agent sequentially chooses a query point $x_t \in \mathcal{X} \subset \Real^d$, and obtains a noisy evaluation
\begin{align*}
    y_t = f_t(x_t) + \eta_t,
\end{align*}
where $f_t:\mathcal{X} \to \Real$ is the unknown function that could vary with time, and $\eta_t$ is the zero-mean noise. In particular, the query point $x_t$ is chosen based on query points and rewards up to time $t-1$. The noise sequence $\{\eta_s\}_{s=1}^\infty$ is conditionally $R$-sub-Gaussian for a fixed $R\ge0$, i.e.,
\begin{align*}
    \forall t \ge 1,\text{ } \forall \lambda \in \Real, \text{ } \ex{e^{\lambda \eta_t} \mid \mathcal{F}_{t-1}} \le \exp(\frac{\lambda^2 R^2}{2}),
\end{align*}
where $\mathcal{F}_{t-1} = \sigma(\{x_{\tau},y_{\tau}\}_{\tau=1}^{t-1},x_t)$ is the $\sigma$-algebra generated by the events so far.

\textbf{Regularity Assumptions.} In contrast to~\cite{bogunovic2016time}, we consider a non-Bayesian regularity condition in this paper. That is, for each time $t$, $f_t$ is allowed to be an arbitrary function in a RKHS with a bounded norm. Specifically, the RKHS is denoted by $\mathcal{H}_k(\mathcal{X})$, which is completely determined by the corresponding chosen kernel function $k: \mathcal{X} \times \mathcal{X} \to \Real$. For any function $g \in \mathcal{H}_k(\mathcal{X})$, it satisfies the \emph{reproducing property}: $g(x) = \inner{g}{k(\cdot,x)}_{\mathcal{H}}$, where $\inner{\cdot}{\cdot}_{\mathcal{H}}$ is the inner product defined on $\mathcal{H}_k(\mathcal{X})$. We assume that $\norm{f_t}_{\mathcal{H}} \le B$ for all $t \ge 1$, and the domain $\mathcal{X}$ is compact. The kernel function $k$ is a continuous kernel with respect to a finite Borel measure $\nu$ whose support is $\mathcal{X}$, and $k(x,x) \le 1$. The assumptions hold for practically relevant kernels (cf.~\cite{riutort2020practical}).

\textbf{Time-varying Assumptions.} We assume that the total variation of $f_t$ satisfies the following budget
\begin{align*}
    \sum_{t=1}^{T-1} \norm{f_{t+1} - f_t}_{\mathcal{H}} \le P_T.
\end{align*}
One nice feature of this variation budget model is that it includes both slowly-changing and abruptly-changing environments. 

The objective is to minimize the following \emph{dynamic regret} given the total variation budget
\begin{align*}
    \mathcal{R}(T):=\sum_{t=1}^T \max_{x\in \mathcal{X}} f_t(x) - f_t(x_t),
\end{align*}
which is the cumulative regret against the optimal strategy that has full information of unknown varying functions.

To design our learning strategy, we will adopt a surrogate model, which helps us to update the estimate and uncertainty of the underlying function after each new query. In particular, we will use a GP prior and Gaussian likelihood, which is specified below. Note that, this surrogate model is only used for the algorithm design. That is, it will not affect the fact that each $f_t$ is an element in RKHS and the noise is allowed to be sub-Gaussian. This is often termed as the \emph{agnostic} setting~\cite{srinivas2009gaussian}.

\textbf{Surrogate Model.} 
We use $\mathcal{GP}(0,k(\cdot,\cdot))$ as an initial prior on the unknown black-box function $f_t$, and a Gaussian likelihood with the noise variables $\eta_t$ drawn independently across $t$ from $\mathcal{N}(0,\lambda)$. Conditioned on a set of observations $\mathcal{H}_t = \{(x_s,{y}_s), s\in[t] \}$, by the properties of GPs~\cite{rasmussen2003gaussian}, the posterior distribution over $f$ is $\mathcal{GP}(\mu_t(\cdot),k_t(\cdot,\cdot))$, where 
		\begin{align}
		\mu_t(x) &= k_t(x)^T(K_t + \lambda I)^{-1}Y_t\label{eq:mu}\\
		k_t(x,x')&=k(x,x')-k_t(x)^T(K_t + \lambda I)^{-1}k_t(x')\nonumber\\
		\sigma_t^2(x) &= k_t(x,x)\label{eq:sigma},
	\end{align}
	in which $Y_t$ is the reward vector $ [y_1,y_2,\ldots,y_t]^T$, $k_t(x)=[k(x_1,x),\ldots,k(x_t,x)]^T$, and $K_t=[k(u,v)]_{u,v\in \mathcal{H}_t}$.
	Therefore, for every $x\in\mathcal{X}$, the posterior distribution of $f(x)$, given $\mathcal{H}_t$ is $\mathcal{N}(\mu_t(x),\sigma_t^2(x))$. The following term (named maximum
information gain) often plays a key role in the regret bounds of GP based algorithms.
	\begin{align}
	\label{eq:ig}
		\gamma_t :=\gamma_t(k,\mathcal{X}) = \max_{A \subset \mathcal{X}: |A| = t} \frac{1}{2}\ln |I_t + {\lambda}^{-1}K_A |,
	\end{align}
	where $K_A = [k(x,x')]_{x,x'\in A}$. It is a function of the kernel $k$ and domain $\mathcal{X}$. For instance, if $\mathcal{X}$ is compact and convex, then we have $\gamma_t = O((\ln t)^{d+1})$ for $k_\text{SE}$, $O(t^{\frac{d(d+1)}{2\nu+d(d+1)}}\ln t)$ for $k_{\text{Mat\'ern}}$, and $O(d\ln t)$ for a linear kernel~\cite{srinivas2009gaussian}.

\section{Main Results}
In this section, we first introduce two Bayesian optimization algorithms based on GP-UCB for time-varying environments, followed by the corresponding upper bounds on the dynamic regret.

\subsection{Algorithms}
The two algorithms basically implement the restart and sliding window mechanism in the GP setting. Specifically, 
the first algorithm (Algorithm~\ref{alg:R}) is simply restart a GP-UCB type algorithm ((i.e., IGP-UCB in~\cite{chowdhury2017kernelized})) every $H$ time steps, named R-GP-UCB. The intuition is clear: since the environment is chaning, we might need to discard our old estimates and restart to build our new estimates.  The second algorithm (Algorithm~\ref{alg:SW}) is based on the sliding window technique. That is, it only uses the most recent $w$ samples for estimations. 

As we mentioned before, the design ideas behind both algorithms are not new as they have been considered in the linear bandit or the Gaussian process bandit setting. Our key contribution is to derive the first (frequentist) regret guarantee on the dynamic regret in the agnostic GP setting.


\begin{algorithm}[t!]
\caption{R-GP-UCB}
\label{alg:R}
\DontPrintSemicolon
\KwIn{Prior $\mathcal{GP}(0,k)$, parameters $B, R, \lambda, \delta$, reset interval $H$}
Initialization: $\beta_0 = 0$, $\mu(x) = 0$ and $\sigma_0(x) = 0 \text{ } \forall x \in \mathcal{X}$\;
\For{$t=1,2,3,\ldots,T$}{
\If{ $t \bmod  H = 1$ }{
    Reset to the initialization state, i.e., $t_0 = t$\;
}
set $\beta_t = B + R\sqrt{2(\gamma_{t-t_0}+1+\ln(1/\delta))} $ \;
$x_t = \argmax_{x \in \mathcal{X}} \mu_{t-1}(x) + \beta_t\sigma_{t-1}(x)$\;
choose $x_t$, observe reward $y_t$ \;
Use only the samples since $t_0$ to update $\mu_t$ and $\sigma_t$ via~\eqref{eq:mu} and \eqref{eq:sigma}\;
}
\end{algorithm}

\subsection{Regret Bounds}
In this section, we show that with a proper choice of the algorithm parameters (reset interval $H$ and sliding window size $w$), both algorithms can achieve sub-linear dynamic regrets. 
\begin{thm}
\label{thm:R}
R-GP-UCB with a reset interval $H$ achieves a high probability regret bound for probability parameter $\delta \in (0,1)$,
\begin{align*}
     \mathcal{R}(T) = O\left( \sqrt{\gamma_H}H^{3/2}P_T + \beta(\delta)T\sqrt{\frac{\gamma_H}{H}}\right),
\end{align*}
where $\beta(\delta) = \left(B + \frac{1}{\sqrt{\lambda}}R\sqrt{2\gamma_{H} + 2\ln(1/\delta)}  \right)$.
\end{thm}

\begin{thm}
\label{thm:S}
SW-GP-UCB with a window size $w$ achieves a high probability regret bound for probability parameter $\delta \in (0,1)$,
\begin{align*}
     \mathcal{R}(T) = O\left( \sqrt{\gamma_w}w^{3/2}P_T + \beta_T(\delta) T\sqrt{\frac{\gamma_w}{w}}\right),
\end{align*}
$\beta_T(\delta) = \left(B + \frac{1}{\sqrt{\lambda}}R\sqrt{2\gamma_{w} + 2\ln(T/\delta)}  \right)$.
\end{thm}
\noindent
\textbf{Remark.} (1) From Theorems~\ref{thm:R} and~\ref{thm:S}, we can see R-GP-UCB and SW-GP-UCB share similar regret bounds, with the difference being the $\beta$ term. The additional $T$ factor in $\beta_T(\delta)$ comes from the information loss due to the sliding window. (2) In the case of a known $P_T$, one can set $H ( \text{or } w) = \widetilde{O}(\gamma_T^{1/4} (T/P_T)^{1/2} )$ to achieve a dynamic regret $\widetilde{O}(\gamma_T^{7/8} (1+P_T)^{1/4} T^{3/4})$. (3) For an unknown $P_T$, one can set $H ( \text{or } w) = \widetilde{O}( \gamma_T^{1/4} T^{1/2} )$ to achieve a dynamic regret $\widetilde{O}(\gamma_T^{7/8} (P_T  + 1) T^{3/4})$. (4) These results directly recover the results for linear bandits by choosing a linear kernel, the $\gamma_T$ of which is $O(d\ln T)$.

\section{Proofs of Theorems}
In this section, we present the proofs for Theorems~\ref{thm:R} and~\ref{thm:S}. Compared to  previous works on linear bandits, the main challenge of our proofs is that the feature map associated with the kernel function could be infinite dimension. As a result, previous finite-dimension results cannot be directly applied since the regret bounds there grow to infinity with $d$. To this end, one possibility is resort to the general operator theory and establish results that hold for a general separable Hilbert space as in~\cite{abbasi2013online}. Here, instead we will directly focus on RKHS and derive our results with an explicit (infinite) feature map enabled by Mercer's theorem. Note that, even this explicit feature map is not necessary (e.g., one can directly use $k(\cdot,x)$ as an implicit feature map as in~\cite{chowdhury2017kernelized,zhou2020local}). We choose to directly focus on RKHS and use the explicit feature map because it directly helps to reveal the connections between linear bandits and Gaussian process bandits.

\begin{algorithm}[t!]
\caption{SW-GP-UCB}
\label{alg:SW}
\DontPrintSemicolon
\KwIn{Prior $\mathcal{GP}(0,k)$, parameters $B, R, \lambda, \delta$, window size $w$.}
Initialization: $\beta_0 = 0$, $\mu(x) = 0$ and $\sigma_0(x) = 0 \text{ } \forall x \in \mathcal{X}$\;
\For{$t=1,2,3,\ldots,T$}{
set $\beta_t = B + R\sqrt{2(\gamma_{t \wedge w}+1+\ln(1/\delta))} $ \;
$x_t = \argmax_{x \in \mathcal{X}} \mu_{t-1}(x) + \beta_t\sigma_{t-1}(x)$\;
choose $x_t$, observe reward $y_t$ \;
Use only the samples between $t_0 = 1 \vee (t-w)$ and $t$ to update $\mu_t$ and $\sigma_t$ via~\eqref{eq:mu} and \eqref{eq:sigma}\;
}
\end{algorithm}

\subsection{Mercer Representation}
The following version of Mercer's theorem is adapted from Theorem 4.1 and 4.2 in~\cite{kanagawa2018gaussian}, which roughly says that the kernel function can be expressed in terms of the eigenvalues and eigenfunctions under mild conditions.
\begin{thm}
Let $\mathcal{X}$ be a compact metric space, $k:\mathcal{X} \times \mathcal{X} \to \Real$ be a continuous kernel with respect to a finite Borel measure $\nu$ whose support is $\mathcal{X}$. Then, there is an at most countable sequence $(\lambda_i,\phi_i)_{i\in \mathbb{N}}$, where $\lambda_i \ge 0$ and $\lim_{i\to \infty}\lambda_i = 0$ and $\{\phi_i\}$ forms an an orthonormal basis of $L_{2,\nu}(\mathcal{X})$, such that 
\begin{align*}
    k(x,x') = \sum_{i\in \mathbb{N}}\lambda_i \phi_i(x)\phi_i(x'), \quad x,x'\in \mathcal{X},
\end{align*}
where the convergence is absolute and uniform over $x,x' \in \mathcal{X}$. Further, the RKHS of $k$ is given by 
\begin{align*}
    \mathcal{H} = \left\{f = \sum_{i\in \mathbb{N}} \theta_i \sqrt{\lambda_i}\phi_i : \norm{f}_{\mathcal{H}} := \sum_{i\in \mathbb{N}} \theta_i^2 < \infty  \right\},
\end{align*}
and the inner-product is given by $\inner{f}{g}_{\mathcal{H}} = \sum_{i\in \mathcal{N} } \alpha_i\beta_i$,
for $f=\sum_{i\in\mathcal{N} }\alpha_i\sqrt{\lambda_i}\phi_i $ and $f=\sum_{i\in\mathcal{N} }\beta_i\sqrt{\lambda_i}\phi_i $.
\end{thm}

Based on this result, we can explicitly define a feature map as $\varphi(x) = (\varphi_1(x),\varphi_2(x), \ldots)$ where $\varphi_i = \sqrt{\lambda_i}\phi_i$. Given a $\theta = (\theta_1,\theta_2,\ldots)$, we denote $\theta^T\varphi(x) = \inner{\theta}{\varphi(x)}$ for $\sum_{i \in \mathcal{N} } \theta_i \varphi_i(x)$ similar to the finite case. Therefore, we have $k(x,y) = \varphi(x)^T \varphi(y)$, and more importantly, $f(x) = \theta^T \varphi(x)$ for $f = \sum_{i\in \mathcal{N}} \theta_i \varphi_i$ and $\norm{f}_{\mathcal{H}}^2 = \norm{\theta}^2 := \sum_{i\in\mathcal{N} } \theta_i^2$. From this, one can easily see that the linear bandit in the finite dimension is just a special case of this general formulation.

We will also introduce some basic notations by using the feature map. For any $f_s \in \mathcal{H}$, we let $\theta_s$ denotes its corresponding parameter such that $\norm{f}_{\mathcal{H}} = \norm{\theta_s}_2$.
Given a set $\{x_1,x_2,\ldots, x_t\} \in \mathcal{X}$, we define a $t \times \infty$ matrix\footnote{Note that, in most cases of our proof, we can safely use linear algebra for the infinite matrices that come up, since all of them are compact operators (and hence a well-established spectral theory). Some are even Hilbert-Schmidt or trace-class operators, which can be arbitrarily approximated by sufficient `large' matrices, see~\cite{miller2019introduction}.} $\Phi_t$ such that $\Phi_t^T = (\varphi(x_1), \varphi(x_2), \ldots, \varphi(x_t))$. We also introduce the (infinite) design matrix $V_t = \Phi_t^T \Phi_t + \lambda I$  and noise vector $N_t = (\eta_1,\ldots, \eta_t)$. For a positive definite matrix $V$, We also define the inner product $\inner{\cdot}{\cdot}_V : =\inner{\cdot}{V\cdot}$ with the corresponding norm as $\norm{\cdot}_V$.

\subsection{Proof of Theorem~\ref{thm:R}}
In this section, we present the proof for Theorem~\ref{thm:R}. Comparing this proof with the linear bandit case, one can gain more insight on the connections between them. 

We will first establish the following bound on the estimate $\mu_t$ in terms of $\sigma_t$ under R-GP-UCB.

\begin{lemma}
\label{lem:R_est}
Let $t_0$ be the most recent restart time before a given time $t \ge t_0$, then for any $\delta \in (0,1)$, with probability at least $1-\delta$, the following holds for any $x\in \mathcal{X}$ and any $t\ge t_0$,
\begin{align*}
    |\mu_{t-1}(x) - f_t(x)| \le &\frac{1}{\lambda}\sqrt{H 2(1+\lambda)\gamma_H}\sum_{s=t_0}^{t-1} \norm{f_s -f_{s+1}}_{\mathcal{H}} \\
    &+ \beta_t \sigma_{t-1}(x),
\end{align*}
where $\beta_t = \left(B + \frac{1}{\sqrt{\lambda}}R\sqrt{2\gamma_{t-t_0} + 2\ln(1/\delta)}  \right)$.
\end{lemma}

Now, let $x_t^* = \max_{x\in \mathcal{X}}f_t(x)$  and $\xi_{H}:=\frac{1}{\lambda}\sqrt{H 2(1+\lambda)\gamma_H}$. Then, based on Lemma~\ref{lem:R_est}, for any $\delta \in (0,1)$, we have with probability at least $1-\delta$,
\begin{align*}
    &r_t = f_t(x_t^*) - f_t(x_t)\\
    &\lep{a} \mu_{t-1}(x_t^*) + \xi_H\sum_{s=t_0}^{t-1} \norm{f_s -f_{s+1}}_{\mathcal{H}} + \beta_t \sigma_{t-1}(x_t^*) -f_t(x_t)\\
    &\lep{b} \mu_{t-1}(x_t) + \xi_H\sum_{s=t_0}^{t-1} \norm{f_s -f_{s+1}}_{\mathcal{H}} + \beta_t \sigma_{t-1}(x_t) -f_t(x_t)\\
     &\lep{c} 2\xi_H\sum_{s=t_0}^{t-1} \norm{f_s -f_{s+1}}_{\mathcal{H}} + 2\beta_t \sigma_{t-1}(x_t)\\
     &\ep{d}2\xi_H\sum_{s=t_0}^{t-1} \norm{f_s -f_{s+1}}_{\mathcal{H}} + 2\sqrt{\lambda}\beta_t \norm{\varphi(x_t)}_{V_{t-1}^{-1}}\\
     &\lep{e}2\xi_H\sum_{s=t_0}^{t-1} \norm{f_s -f_{s+1}}_{\mathcal{H}} + 2\sqrt{\lambda}\beta(\delta) \norm{\varphi(x_t)}_{V_{t-1}^{-1}}
\end{align*}
where (a) and (c) follow from Lemma~\ref{lem:R_est}; (b) follows from the UCB-type algorithm; (d) holds by Claim~\ref{clm:R_1}, where $V_t = \lambda I + \Phi_t^T \Phi_t$ and $\Phi_t^T = (\varphi(x_{t_0}),  \ldots, \varphi(x_t))$, $t_0$ is the most recent rest time slot before $t$; in (e), $\beta(\delta) = \left(B + \frac{1}{\sqrt{\lambda}}R\sqrt{2\gamma_{H} + 2\ln(1/\delta)}  \right)$.
Thus, we have the following regret bound,
\begin{align*}
    &\mathcal{R}(T) \le 2H^{3/2}P_T\sqrt{\gamma_H} + \sqrt{\lambda}\beta(\delta)\sum_{t=1}^T\norm{\varphi(x_t)}_{V_{t-1}^{-1}}.
\end{align*}
The sum $\sum_{t=1}^T\norm{\varphi(x_t)}_{V_{t}^{-1}}$ needs further analysis. Note that we can divide the time horizon into blocks of size $H$, and reset starts at the beginning of each block.
\begin{align*}
    \sum_{t=1}^T\norm{\varphi(x_t)}_{V_{t}^{-1}} \le \sum_{k=0}^{T/H - 1} \sum_{t=kH+1}^{(k+1)H} \norm{\varphi(x_t)}_{V_{t-1}^{-1}},
\end{align*}
where for all $t \in [kH+1,(k+1)H]$, $t_0 = kH+1$. Thus, for each block, we have by Lemma 4 in~\cite{chowdhury2017kernelized}, 
\begin{align*}
    {\sqrt{\lambda}} \sum_{t=kH+1}^{(k+1)H} \norm{\varphi(x_t)}_{V_{t-1}^{-1}} \le \sqrt{4(H+2)\gamma_H}.
\end{align*}

Putting everything together, yields $\mathcal{R}(T) = O\left( \sqrt{\gamma_H}H^{3/2}P_T + \beta(\delta)T\sqrt{\frac{\gamma_H}{H}}\right).$
    

\begin{proof}[Proof of Lemma~\ref{lem:R_est} ]
By a slight abuse of notation, here $k_t(x) = [k(x_{t_0},x),\ldots, k(x_t,x)]^T$ and $K_t=[k(x_u,x_v)]_{u,v\in \{t_0,\ldots,t\} }$. That is, we only use the samples starting from the most recent reset to calculate $\mu_{t}$ and $\sigma_t$. Let $\Phi_t^T = (\varphi(x_{t_0}),\ldots, \varphi(x_t))$, then the mean estimate $\mu_t(x)$ can be rewritten as
\begin{align*}
    \mu_{t}(x) &= k_t(x)^T(K_t + \lambda I)^{-1}Y_t\\
    &=[\Phi_t \varphi(x)]^T(\Phi_t\Phi_t^T + \lambda)^{-1}Y_t\\
    &=\inner{\varphi(x)}{(\Phi^T\Phi_t + \lambda I)^{-1}\Phi_t^T Y_t  },
\end{align*}
where the last equality we have used the fact $(A^TA+\lambda I)^{-1}A^T = A^T(AA^T + \lambda I)^{-1}$, which holds by Lemma 3 of~\cite{chowdhury2019bayesian}.
Recall that $f_t(x) = \inner{\theta_t}{\varphi(x)}$ where $\theta_t$ is the parameters of $f_t$. Combining this with the result above, yields $|f_t(x) - \mu_{t-1}(x)| = |\inner{\varphi(x)}{\theta_t - \hat{\theta}_{t-1}}|$,
where $\hat{\theta}_{t-1} = V_{t-1}^{-1}\Phi_{t-1}^T Y_{t-1}$.  Note that, the term $\hat{\theta}_{t-1} - \theta_t$ can be rewritten as
\begin{align*}
    &\hat{\theta}_{t-1} - \theta_t\\
    =&V_{t-1}^{-1}\left(\sum_{s=t_0}^{t-1}\varphi(x_s)\varphi(x_s)^T(\theta_s -\theta_t)  \right) + V_{t-1}^{-1}\Phi_{t-1}^T N_{t-1} \\
    &- \lambda V_{t-1}^{-1} \theta_{t}.
\end{align*}
Thus, by the triangle inequality, we have  
\begin{align}
    &|f_t(x) - \mu_{t-1}(x)|\nonumber\\
    \le &\left|\inner{\varphi(x)}{\Phi_{t-1}^T N_{t-1}}_{V_{t-1}^{-1}}\right| + \lambda \left|\inner{\varphi(x)}{\theta_t}_{V_{t-1}^{-1}}\right|\label{eq:R_1}\\
    &+ \left|\inner{\varphi(x)}{V_{t-1}^{-1}\left(\sum_{s=t_0}^{t-1}\varphi(x_s)\varphi(x_s)^T(\theta_s -\theta_t)  \right) } \right|.\label{eq:R_2}
\end{align}
From the above, we can see \eqref{eq:R_2} is the additional term due to time-varying environments. We now turn to bound each term, respectively. The second term in~\eqref{eq:R_1} can be easily bounded under the boundedness assumptions.
\begin{align}
    \lambda \left|\inner{\varphi(x)}{\theta_t}_{V_{t-1}^{-1}}\right| &\le \lambda \norm{\varphi(x)}_{V_{t-1}^{-1}}\norm{V_{t-1}^{-1/2} \theta_t }\nonumber\\
    &\le \sqrt{\lambda} B \norm{\varphi(x)}_{V_{t-1}^{-1}}\label{eq:R_Bterm},
\end{align}
where the last inequality follows from $V_{t-1}^{-1} \preceq \lambda^{-1} I$, and $\norm{\theta_t} = \norm{f_t}_{\mathcal{H}} \le B$. For the first term in~\eqref{eq:R_1}, we can bound it by using the RKHS-valued self-normalized inequality (e.g., Lemma 7 in~\cite{chowdhury2019bayesian}). First, let $\widetilde{V}_{t-1} = V_{t-1}/\lambda$, and hence we have\footnote{Note that the main purpose of defining $\widetilde{V}_{t-1}$ here is to make sure that $\det(\widetilde{V}_{t-1})$ is well-defined. This is because $\frac{1}{\lambda} \Phi_t^T \Phi_t$ is a trace-class operator, and hence $\det(I +\frac{1}{\lambda} \Phi_t^T \Phi_t )$ is well defined via Fredholm determinant.}
\begin{align}
    &\left|\inner{\varphi(x)}{\Phi_{t-1}^T N_{t-1}}_{V_{t-1}^{-1}}\right|\nonumber\\
    =& \frac{1}{\lambda} \norm{\varphi(x)}_{\widetilde{V}_{t-1}^{-1}} \norm{\sum_{s=t_0}^{t-1} \eta_s \varphi(x_s) }_{\widetilde{V}_{t-1}^{-1}}\nonumber\\
    \le&\frac{1}{\lambda} \norm{\varphi(x)}_{\widetilde{V}_{t-1}^{-1}}\sqrt{2R^2\lambda\ln\left(\det(\widetilde{V}_{t-1})^{1/2} /\delta \right)}\label{eq:R_Nterm}.
\end{align}

Then, by applying Claim~\ref{clm:R_1} below to \eqref{eq:R_Bterm} and~\eqref{eq:R_Nterm}, we can upper bound~\eqref{eq:R_1} as
\begin{align}
    \eqref{eq:R_1} \le \sigma_{t-1}(x)\left(B + \frac{1}{\sqrt{\lambda}}R\sqrt{2\gamma_{t-t_0} + 2\ln(1/\delta)}  \right)\label{eq:bound_R1}.
\end{align}
\begin{claim}
\label{clm:R_1}
    The following equations hold for all $t \ge t_0$.
    \begin{align*}
        &\lambda \norm{\varphi(x)}_{V_{t}^{-1}}^2 = \sigma_t^2(x)\\
        &\norm{\varphi(x)}_{\widetilde{V}_{t}^{-1}}^2 = \sigma_t^2(x)\\
        &\ln(\det(\widetilde{V}_t))  \le 2\gamma_{t-t_0}\le 2\gamma_H.
    \end{align*}
\end{claim}
Now, to bound~\eqref{eq:R_2}, we will follow similar arguments in~\cite{cheung2019learning}. In particular, we have 
\begin{align}
    \eqref{eq:R_2} &\lep{a} \norm{V_{t-1}^{-1}\left(\sum_{s=t_0}^{t-1}\varphi(x_s)\varphi(x_s)^T(\theta_s -\theta_t)  \right)}\nonumber\\
    &=\norm{V_{t-1}^{-1}\left(\sum_{p=t_0}^{t-1}\sum_{s=t_0}^{p}\varphi(x_s)\varphi(x_s)^T( \theta_p - \theta_{p+1} )  \right)}\nonumber\\
    &\le\sum_{p=t_0}^{t-1}\norm{V_{t-1}^{-1} \sum_{s=t_0}^{p}\varphi(x_s)\varphi(x_s)^T( \theta_p - \theta_{p+1} )  }\nonumber\\
    &\lep{b}  \frac{1}{\lambda}\sqrt{H 2(1+\lambda)\gamma_H}\sum_{p=t_0}^{t-1}\norm{f_p - f_{p+1}}_{\mathcal{H}}\label{eq:bound_R2},
\end{align}
where (a) follows from $\norm{\varphi(x)} = \sqrt{k(x,x)} \le 1$; (b) follows from the following claim and the fact that $\norm{\theta_p -\theta_{p+1}} = \norm{f_p -f_{p+1}}_{\mathcal{H}}$.
\begin{claim}
\label{clm:R_2}
    For any $t_0 \le p \le t-1$, the operator norm satisfies
    \begin{align*}
        \norm{V_{t-1}^{-1} \sum_{s=t_0}^{p}\varphi(x_s)\varphi(x_s)^T } \le \frac{1}{\lambda}\sqrt{2H (1+\lambda)\gamma_H}.
    \end{align*}
\end{claim}

Combining the bounds in~\eqref{eq:bound_R1} and~\eqref{eq:bound_R2}, yields the final result of Lemma~\ref{lem:R_est}. We are left to present proofs for the claims. The results in the first claim are standard (cf. Appendix in~\cite{chowdhury2019bayesian}). We give a proof for Claim~\ref{clm:R_2}.

\textbf{Proof of Claim~\ref{clm:R_2}.} Previous works (including~\cite{cheung2019learning,russac2019weighted,zhao2020simple}) all bound the LHS of Claim~\ref{clm:R_2} by one, which is not true as shown in~\cite{zhao2021non}. To address this gap, we follow the key idea in~\cite{zhao2021non} with slight modifications in step (c) below to handle the possibly infinite dimension in our case.

Denote the unit ball $\mathcal{B}(1) = \{z | \norm{z}=1 \}$. Then, we have
\begin{align*}
    &\norm{V_{t-1}^{-1} \sum_{s=t_0}^{p}\varphi(x_s)\varphi(x_s)^T } \\
    &=\sup_{z \in \mathcal{B}(1)} \left|z^TV_{t-1}^{-1} \left( \sum_{s=t_0}^{p}\varphi(x_s)\varphi(x_s)^T \right)z \right|\\
    &\lep{a} \norm{z^*}_{V_{t-1}^{-1}}\norm{\left( \sum_{s=t_0}^{p}\varphi(x_s)\varphi(x_s)^T \right)z}_{V_{t-1}^{-1}}\\
    &\le \norm{z^*}_{V_{t-1}^{-1}}\norm{ \sum_{s=t_0}^{p}\varphi(x_s) \norm{\varphi(x_s)} \norm{z}}_{V_{t-1}^{-1}}\\
    &\le \frac{1}{\sqrt{\lambda}} \norm{\sum_{s=t_0}^{p}\varphi(x_s)}_{V_{t-1}^{-1}}\\
    &\lep{b}  \frac{1}{\sqrt{\lambda}} \sum_{s=t_0}^{p} \norm{\varphi(x_s)}_{V_{t-1}^{-1}}\\
    &\lep{c} \frac{1}{\lambda} \sum_{s=t_0}^{p} \sigma_{s-1}(x_s)\\
    &\lep{d} \frac{1}{\lambda} \sqrt{H \sum_{s=t_0}^{p}\sigma_{s-1}^2(x_s) }\\
    &\lep{e} \frac{1}{\lambda}\sqrt{H 2(1+\lambda)\gamma_H}
\end{align*}
where in (a) $z^*$ is the optimizer; (b) holds by $\norm{x}_{V_{t-1}^{-1}} \le \norm{x}/\sqrt{\lambda}$ and $\norm{z} = 1$ and $\norm{\varphi(x_s)}  \le 1$; (c) follows from the fact that $V_{t-1}^{-1} \preceq V_{s-1}^{-1}$ and Claim~\ref{clm:R_1}; (d) follows from Cauchy–Schwarz inequality since $p-t_0$ is at most $H$; (e) holds by $\sum_{s=t_0}^{p}\sigma_{t-1}^2(x_s) \le 2(1+\lambda)\gamma_H$ by Lemma 6 in~\cite{chowdhury2019bayesian}. Note that this holds because $\sigma_{t-1}$ in our case is only updated via data points starting from $t_0$ and the fact that $p-t_0 \le H$.

\textbf{Proof of Claim~\ref{clm:R_1}.} We begin with the first equation. Note that $(\Phi_t^T \Phi_t + \lambda I) \varphi(x) = \Phi_t^T k_t(x) + \lambda \varphi(x)$. Combining this with the fact that $(\Phi_t^T\Phi_t+\lambda I)^{-1}\Phi_t^T = \Phi_t^T(\Phi_t\Phi_t^T + \lambda I)^{-1}$, yields
\begin{align*}
    &\varphi(x)\\
    =& \Phi_t^T(\Phi_t\Phi_t^T + \lambda I)^{-1}k_t(x) + \lambda(\Phi_t^T\Phi_t+\lambda I)^{-1}\varphi(x).
\end{align*}
This directly leads to 
\begin{align*}
    &\lambda \norm{\varphi(x)}_{V_{t}^{-1}}^2\\
    =&\lambda \varphi(x)^T(\Phi_t^T\Phi_t+\lambda I)^{-1}\varphi(x) \\
    = &k(x,x) - k_t(x)^T(K_t+\lambda I)^{-1}k_t(x)\\
    =&\sigma_t^2(x).
\end{align*}
which proves the first equation. The second equation follows directly from the first one. To prove the third equation, first note that
\begin{align*}
    \ln\det(\widetilde{V}_t) &=  \ln\det(I + \frac{1}{\lambda} \Phi_t^T\Phi_t  ) \\
    &=\ln\det( I + \frac{1}{\lambda}\Phi_t\Phi_t^T )\\
    &=\ln\det(I+\frac{1}{\lambda}K_t ).
\end{align*}
Recall the definition of $\gamma_t$ in~\eqref{eq:ig}, we now have \begin{align*}
    \ln\det(\widetilde{V}_t) \le 2 \gamma_{t-t_0} \le 2\gamma_{H}.
\end{align*}

\end{proof}

\subsection{Proof of Theorem~\ref{thm:S}}
The proof shares great similarity with the proof of Theorem~\ref{thm:R}. However, as pointed out by~\cite{russac2019weighted}, due to the sliding window feature, one cannot directly apply the standard self-normalized inequality. We adopt the same trick proposed in~\cite{russac2019weighted} to handle this. 
The challenge here is to handle the possibly infinite dimension. To this end, we will apply Fatou's lemma.
We first have the following bound. Due to the information loss by sliding window, there is a $T$ factor in the $\beta$ term.
\begin{lemma}
\label{lem:S_est}
For any $\delta \in (0,1)$, with probability at least $1-\delta$, the following holds for any $x\in \mathcal{X}$ and any $t\ge 1$,
\begin{align*}
    |\mu_{t-1}(x) - f_t(x)| \le \sum_{s=t_0}^{t-1} \norm{f_s -f_{s+1}}_{\mathcal{H}} + \beta_t \sigma_{t-1}(x),
\end{align*}
where $\beta_t = \left(B + \frac{1}{\sqrt{\lambda}}R\sqrt{2\gamma_{t \wedge w} + 2\ln(T/\delta)}  \right)$.
\end{lemma}
\textbf{Proof.} The key difference compared to Lemma~\ref{lem:R_est} is a new bound on $\norm{\sum_{s=t_0}^{t-1} \eta_s \varphi(x_s) }_{\widetilde{V}_{t-1}^{-1}}$, where $t_0 = \max(1,t-w)$. We denote $S_{t-1} = \sum_{s=t_0}^{t-1} \eta_s \varphi(x_s) $, following the trick in~\cite{russac2019weighted} to handle the information loss due to sliding window, we further define $\widehat{V}_{u-1} = \sum_{s=1 \vee t-w }^{u-1} (1/\lambda) \varphi(x_s)\varphi(x_s)^T  + I$,  $\widehat{S}_{u-1} = \sum_{s=1 \vee t-w }^{u-1}\eta_s \varphi(x_s)$ and $\widehat{M}_{u-1}(q) = \exp(\frac{1}{\sqrt{\lambda}}q^T \widehat{S}_{u-1} - \frac{R^2}{2\lambda} q^T\widehat{V}_{u-1}(0)q)$, where $\widehat{V}_{u-1}(0) = \widehat{V}_{u-1} - I$. Note that by these definitions, $\widehat{S}_{t-1} = S_{t-1}$ and $\widehat{V}_{t-1} = \widetilde{V}_{t-1}$.
We then take an infinite Gaussian random sequence (independent of all other randomness) $Q \sim N(0,I)$, and define $M_{u-1} = \mathbb{E}\left[{\widehat{M}_{t-1}(Q)}\right]$. By standard martingale arguments and sub-Gaussian noise assumption, we have $\ex{M_{t-1}} \le 1$. We now take a finite approximation by using the first $d$ dimension of the feature map. In particular, we denote by $Q_d, M_{u-1,d}$, $\widehat{S}_{u-1,d}$ and $\widehat{V}_{u-1,d}$ truncated versions. Clearly, $ \ex{M_{t-1,d}} \le 1 $ for all $d$. Moreover, we can take the finite-version result for $M_{t-1,d}$ in~\cite{russac2019weighted}. That is, $M_{t-1,d} =\frac{1}{\det(\widetilde{V}_{t-1,d})^{1/2}}\exp(\frac{1}{2R^2\lambda}\norm{S_{t-1,d}}_{\widetilde{V}_{t-1,d}^{-1}}^2 )$. Now, by using Fatou's lemma, we can obtain that 
\begin{align*}
    &\prob{\lim_{d \to \infty} \frac{ { \norm{S_{t-1,d}}_{\widetilde{V}_{t-1,d}^{-1}}^2 } }{2R^2\lambda \ln(\det(\widetilde{V}_{t-1,d})^{1/2} /\delta) } \ge 1}\\
    \le &\ex{\lim_{d\to \infty}  \frac{\delta\exp(\frac{1}{2R^2\lambda}\norm{S_{t-1,d}}_{\widetilde{V}_{t-1,d}^{-1}}^2 )}{\det(\widetilde{V}_{t-1,d})^{1/2}}} \\
    \le& \delta \lim_{d\to \infty} \ex{M_{t-1,d}} \le \delta.
\end{align*}
Due to the sliding window, we cannot apply the the standard `stopping time' trick. Instead, we use union bound to obtain, for any $t\ge 1$, with probability at least $1-\delta$, ${ \norm{S_{t-1}}_{\widetilde{V}_{t-1}^{-1}}^2 } \le 2R^2\lambda \ln(\det(\widetilde{V}_{t-1,d})^{1/2} T/\delta)$.
\qed

Now, following the same steps as in Theorem~\ref{thm:R}, we obtain
\begin{align*}
    \mathcal{R}(T) \le 2HP_T + \sqrt{\lambda}\beta_T(\delta)\sum_{t=1}^T\norm{\varphi(x_t)}_{V_{t-1}^{-1}},
\end{align*}
where $V_{t-1} = \sum_{s = 1\vee t-w}^{t-1} \varphi(x_s)\varphi(x_s)^T + \lambda I$. Similar to the procedure in the proof of Theorem~\ref{thm:R}, we have $\sum_{t=1}^T\norm{\varphi(x_t)}_{V_{t}^{-1}} \le \sum_{k=0}^{T/w - 1} \sum_{t=kw+1}^{(k+1)w} \norm{\varphi(x_t)}_{V_{t-1}^{-1}}.$
Let $W_{t-1}^{(k)} = \sum_{s = kw+1}^{t-1} \varphi(x_s)\varphi(x_s)^T + \lambda I$ (which is the $V_{t-1}$ in the proof of Theorem~\ref{thm:R}), and hence for $t\in[kw,(k+1)w]$, $V_{t-1}^{-1} \preceq W_{t-1}^{(k)}$. Hence, we can now follow the remaining steps in Theorem~\ref{thm:R} to obtain 
\begin{align*}
    \mathcal{R}(T) = O\left( wP_T + \beta_T(\delta) T\sqrt{\frac{\gamma_w}{w}}\right).
\end{align*}

\section{Conclusion}
We studied the black-box bandit optimization under a time-varying environment. We consider the variation  budget  model, which is able to capture both slowly-changing and abruptly-changing environments. We derived the dynamic regret bounds for R-GP-UCB and SW-GP-UCB with a non-Bayesian regularity assumption.




\bibliographystyle{IEEEtran}
\bibliography{references}

\begin{thebibliography}{10}
\providecommand{\url}[1]{#1}
\csname url@samestyle\endcsname
\providecommand{\newblock}{\relax}
\providecommand{\bibinfo}[2]{#2}
\providecommand{\BIBentrySTDinterwordspacing}{\spaceskip=0pt\relax}
\providecommand{\BIBentryALTinterwordstretchfactor}{4}
\providecommand{\BIBentryALTinterwordspacing}{\spaceskip=\fontdimen2\font plus
\BIBentryALTinterwordstretchfactor\fontdimen3\font minus
  \fontdimen4\font\relax}
\providecommand{\BIBforeignlanguage}[2]{{%
\expandafter\ifx\csname l@#1\endcsname\relax
\typeout{** WARNING: IEEEtran.bst: No hyphenation pattern has been}%
\typeout{** loaded for the language `#1'. Using the pattern for}%
\typeout{** the default language instead.}%
\else
\language=\csname l@#1\endcsname
\fi
#2}}
\providecommand{\BIBdecl}{\relax}
\BIBdecl

\bibitem{lai1985asymptotically}
T.~L. Lai and H.~Robbins, ``Asymptotically efficient adaptive allocation
  rules,'' \emph{Advances in applied mathematics}, vol.~6, no.~1, pp. 4--22,
  1985.

\bibitem{abbasi2011improved}
Y.~Abbasi-Yadkori, D.~P{\'a}l, and C.~Szepesv{\'a}ri, ``Improved algorithms for
  linear stochastic bandits,'' in \emph{Advances in Neural Information
  Processing Systems}, 2011, pp. 2312--2320.

\bibitem{srinivas2009gaussian}
N.~Srinivas, A.~Krause, S.~M. Kakade, and M.~Seeger, ``Gaussian process
  optimization in the bandit setting: No regret and experimental design,''
  \emph{arXiv preprint arXiv:0912.3995}, 2009.

\bibitem{chowdhury2017kernelized}
S.~R. Chowdhury and A.~Gopalan, ``On kernelized multi-armed bandits,''
  \emph{arXiv preprint arXiv:1704.00445}, 2017.

\bibitem{auer2019adaptively}
P.~Auer, P.~Gajane, and R.~Ortner, ``Adaptively tracking the best bandit arm
  with an unknown number of distribution changes,'' in \emph{Conference on
  Learning Theory}.\hskip 1em plus 0.5em minus 0.4em\relax PMLR, 2019, pp.
  138--158.

\bibitem{besbes2019optimal}
O.~Besbes, Y.~Gur, and A.~Zeevi, ``Optimal exploration--exploitation in a
  multi-armed bandit problem with non-stationary rewards,'' \emph{Stochastic
  Systems}, vol.~9, no.~4, pp. 319--337, 2019.

\bibitem{cheung2019learning}
W.~C. Cheung, D.~Simchi-Levi, and R.~Zhu, ``Learning to optimize under
  non-stationarity,'' in \emph{The 22nd International Conference on Artificial
  Intelligence and Statistics}.\hskip 1em plus 0.5em minus 0.4em\relax PMLR,
  2019, pp. 1079--1087.

\bibitem{russac2019weighted}
Y.~Russac, C.~Vernade, and O.~Capp{\'e}, ``Weighted linear bandits for
  non-stationary environments,'' \emph{arXiv preprint arXiv:1909.09146}, 2019.

\bibitem{zhao2020simple}
P.~Zhao, L.~Zhang, Y.~Jiang, and Z.-H. Zhou, ``A simple approach for
  non-stationary linear bandits,'' in \emph{International Conference on
  Artificial Intelligence and Statistics}.\hskip 1em plus 0.5em minus
  0.4em\relax PMLR, 2020, pp. 746--755.

\bibitem{zhao2021non}
P.~Zhao and L.~Zhang, ``Non-stationary linear bandits revisited,'' \emph{arXiv
  preprint arXiv:2103.05324}, 2021.

\bibitem{bogunovic2016time}
I.~Bogunovic, J.~Scarlett, and V.~Cevher, ``Time-varying gaussian process
  bandit optimization,'' in \emph{Artificial Intelligence and
  Statistics}.\hskip 1em plus 0.5em minus 0.4em\relax PMLR, 2016, pp. 314--323.

\bibitem{riutort2020practical}
G.~Riutort-Mayol, P.-C. B{\"u}rkner, M.~R. Andersen, A.~Solin, and A.~Vehtari,
  ``Practical hilbert space approximate bayesian gaussian processes for
  probabilistic programming,'' \emph{arXiv preprint arXiv:2004.11408}, 2020.

\bibitem{rasmussen2003gaussian}
C.~E. Rasmussen, ``Gaussian processes in machine learning,'' in \emph{Summer
  School on Machine Learning}.\hskip 1em plus 0.5em minus 0.4em\relax Springer,
  2003, pp. 63--71.

\bibitem{abbasi2013online}
Y.~Abbasi-Yadkori, ``Online learning for linearly parametrized control
  problems,'' 2013.

\bibitem{zhou2020local}
X.~Zhou and J.~Tan, ``Local differential privacy for bayesian optimization,''
  \emph{arXiv preprint arXiv:2010.06709}, 2020.

\bibitem{kanagawa2018gaussian}
M.~Kanagawa, P.~Hennig, D.~Sejdinovic, and B.~K. Sriperumbudur, ``Gaussian
  processes and kernel methods: A review on connections and equivalences,''
  \emph{arXiv preprint arXiv:1807.02582}, 2018.

\bibitem{miller2019introduction}
D.~A. Miller, ``An introduction to functional analysis for science and
  engineering,'' \emph{arXiv preprint arXiv:1904.02539}, 2019.

\bibitem{chowdhury2019bayesian}
S.~R. Chowdhury and A.~Gopalan, ``Bayesian optimization under heavy-tailed
  payoffs,'' in \emph{Advances in Neural Information Processing Systems}, 2019,
  pp. 13\,790--13\,801.

\end{thebibliography}


\end{document}